\documentclass{article}
\usepackage{microtype}
\usepackage{graphicx}
\usepackage{subfigure}
\usepackage{booktabs}
\usepackage{hyperref} 
\usepackage{url}
\usepackage{amsmath,amsthm,amssymb,amsfonts} 
\usepackage{microtype}
\usepackage{pifont}
\usepackage{nicefrac}
\usepackage{mathtools}
\usepackage{balance}
\mathtoolsset{showonlyrefs}
\usepackage{natbib}
\renewcommand{\Pr}{\field{P}}

\newcommand{\bta}{\boldsymbol{\eta}}
\newcommand{\bps}{\boldsymbol{\epsilon}}
\newcommand{\bb}{\boldsymbol{b}}
\newcommand{\ba}{\boldsymbol{a}}
\newcommand{\bm}{\boldsymbol{m}}
\newcommand{\bg}{\boldsymbol{g}}
\newcommand{\bx}{\boldsymbol{x}}

\newcommand{\by}{\boldsymbol{y}}

\newcommand{\field}[1]{\mathbb{#1}}

\newcommand{\R}{\field{R}}

\newcommand{\E}{\field{E}}

\newtheorem{lemma}{Lemma}
\newtheorem{theorem}{Theorem}

\usepackage{hyperref}

\usepackage[accepted]{icml2020}
\icmltitlerunning{A High Probability Analysis of Adaptive SGD with Momentum}

\begin{document}
	
\twocolumn[
\icmltitle{A High Probability Analysis of Adaptive SGD with Momentum}
\icmlsetsymbol{equal}{*}

\begin{icmlauthorlist}
\icmlauthor{Xiaoyu Li}{se}
\icmlauthor{Francesco Orabona}{se,ee}
\end{icmlauthorlist}

\icmlaffiliation{se}{Division of System Engineering, Boston University, MA, USA}
\icmlaffiliation{ee}{Electrical \& Computer Engineering, Boston University, MA, USA}

\icmlcorrespondingauthor{Xiaoyu Li}{xiaoyuli@bu.edu}

\icmlkeywords{Machine Learning, ICML}
\vskip 0.3in
]

\printAffiliationsAndNotice{} 

\begin{abstract}
Stochastic Gradient Descent (SGD) and its variants are the most used algorithms in machine learning applications. In particular, SGD with adaptive learning rates and momentum is the industry standard to train deep networks. Despite the enormous success of these methods, our theoretical understanding of these variants in the nonconvex setting is not complete, with most of the results only proving convergence in expectation and with strong assumptions on the stochastic gradients. In this paper, we present a high probability analysis for adaptive and momentum algorithms, under weak assumptions on the function, stochastic gradients, and learning rates. We use it to prove for the first time the convergence of the gradients to zero in high probability in the smooth nonconvex setting for Delayed AdaGrad with momentum.
\end{abstract}
	
\section{Introduction}
\label{sec:intro}
Despite the incredible popularity of stochastic gradient methods in practical machine learning applications, our theoretical understanding of these methods is still not complete. In particular, \emph{adaptive learning rates} methods like AdaGrad~\citep{DuchiHS11} have been mainly studied in the convex domain, with few analyses in the non-convex domain~\citep{LiO19,WardWB18}. However, even in these latter analyses, the assumptions used are very strong and/or the results limited.

In particular, there are two main problems with the previous analyses of Stochastic Gradient Descent (SGD) and its variants in the nonconvex setting. First, the classic analysis of convergence for SGD in the nonconvex setting uses an analysis in expectation. However, expectation bounds do not rule out extremely bad outcomes. As pointed out by \citet{HarveyLPR19}, it is a misconception that for the algorithms who have expectation bounds it is enough to pick the best of several independent runs to have a high probability guarantee: It can actually be a computational inefficient procedure. Moreover, in practical applications like deep learning, it is often the case that only one run of the algorithm is used since that the training process may take long time. Hence, it is essential to get high probability bounds which guarantee the performance of the algorithm on single runs.

Another very common assumption used in most of the previous papers is the one of bounded stochastic gradients. This is a rather strong assumption and it is false even in the deterministic optimization of a convex quadratic function, e.g., $f(x)=x^2$.

In this work, we overcome both these problems. We prove \emph{high probability convergence rates only assuming that the noises on the gradients are well-behaved}, i.e., subgaussian. In this way, we allow for unbounded gradients and unbounded noise.
The weak assumptions, the nonconvex analysis, and the adaptive learning rates make our results particularly challenging to obtain.
Indeed, high probability bounds for bounded stochastic gradients are almost trivial to obtain but of limited applicability. Overall, we believe this paper is the first one to prove such guarantees.

\paragraph{Contributions.}
In this short paper, we present a high probability analysis of SGD with momentum and adaptive learning rates, with weak assumptions on function and stochastic gradients.
So, first in Theorem~\ref{thm:momentum_sqrt} we prove high probability bounds for the gradients of classic momentum SGD step size $O(\frac{1}{\sqrt{t}})$ in the nonconvex setting. 
Then, in Theorem~\ref{thm:momentum_adagrad} we prove \textit{for the first time} high probability convergence rates for the gradients of AdaGrad with momentum in the nonconvex setting. In particular, we also show that the high probability bounds are adaptive to the level of noise. 

\section{Related Work}
\label{sec:rel}

\textbf{Stochastic momentum methods.} \citet{SutskeverMDH13} discussed the importance of classic momentum methods in deep learning, which is nowadays widely used in the training of neural networks. On the convergence of stochastic momentum methods, \citet{YangLL16} studied a unified momentum method and provided expectation bounds in the rate of $O(\frac{1}{\sqrt{T}})$ in both convex and nonconvex setting. However, the results hold only for Lipschitz functions. \citet{GadatPS18} provided an in-depth description of the stochastic heavy-ball method. Moreover, for the non-convex functions, they showed some almost sure convergence results. \citet{LoizouR17} provided a general analysis for the momentum variants of several classes of stochastic optimization algorithms and proved the linear rate convergence for quadratic and smooth functions. To the best of our knowledge, there are no high probability bounds for nonconvex stochastic momentum methods without using strong assumptions. 

\textbf{Nonconvex convergence of adaptive methods.} In recent years, a variety of adaptive SGD algorithms have been developed to automatically tune the step size by using the past stochastic gradients. The first adaptive algorithm was AdaGrad~\citet{DuchiHS11}, designed to adapt to sparse gradients. \citet{LiO19} and \citet{WardWB18} showed the convergence of variants of AdaGrad with a rate of $O(\ln T /\sqrt{T})$ in the non-convex case. Moreover, \citet{LiO19} showed that AdaGrad with non-coordinate-wise learning rates is adaptive to the level of noise. \citet{ZouSJZL19} studied AdaGrad with a unified momentum and \citet{ChenLSH19} considered a large family of Adam-like algorithms~\citep{KingmaB15} including AdaGrad with momentum. Yet, all of these works prove on bounds in expectation and most of them use the very strong assumption of bounded stochastic gradients. 

\textbf{High probability bounds.} The results on high probability bounds are relatively rare compared to those in expectation, which are easier to obtain. \citet{KakadeT09} used Freeman's inequality to prove high probability bounds for an algorithm solving the SVM objective function. For classic SGD, \citet{HarveyLR19} and \citet{HarveyLPR19} used a generalized Freedman’s inequality to prove bounds in non-smooth and strongly convex case, while \citet{JainNN19} proved the optimal bound for the last iterate of SGD with high probability.  As far as we know, there are currently no high probability bounds for adaptive methods in the nonconvex setting. 

\section{Problem Set-Up}
\label{sec:def}

\paragraph{Notation.}
We denote vectors and matrices by bold letters. The coordinate j of a vector $\bx$ is denoted by $x_j$ and as $\nabla f(\bx)_j$ for the gradient $\nabla f(\bx)$. To keep the notation concise, all standard operations $\bx \by, \bx/\by, \bx^2, 1/\bx, \bx^{1/2}$ and $\max (\bx,\by)$ on the vectors $\bx$, $\by$ are supposed to be element-wise. We denote by $\E [\cdot]$ the expectation with respect to the underlying probability space and by $\E_t[\cdot]$ the conditional expectation with respect to the past randomness. We use $L_2$ norms.

\paragraph{Assumptions.}
In this paper we focus on the optimization problem 
\[
\min_{\bx \in \R^d} \ f(\bx),
\]
where $f$ is bounded from below and we denote its infimum by $f^\star$. We \emph{do not} assume the function to be convex.
We consider stochastic optimization algorithms that have access to a noisy estimate of the gradient of $f$. This covers the ubiquitous SGD~\citep{RobbinsM51}, as well modern variants as AdaGrad. We are interested in studying the convergence of the gradients to zero, because without additional assumptions it is the only thing we can study in the nonconvex setting.

We make the following assumption on the objective function $f(\bx)$:
\begin{itemize}
	\item[(\textbf{A})] $f$ is \emph{$M$-smooth}, that is, $f$ is differentiable and its gradient is $M$-Lipschitz, i.e. $\|\nabla f(\bx)- \nabla f(\by)\| \leq M \|\bx-\by\|, \ \forall \bx, \by \in \R^d$.
\end{itemize}
Note that (\textbf{A}), for all $\bx,\by \in \R^d$, implies ~\citep[Lemma 1.2.3]{Nesterov04}
\begin{equation}
\label{eq:smooth2}
\left|f(\by)-f(\bx)-\langle \nabla f(\bx), \by-\bx\rangle\right|
\leq \frac{M}{2}\|\by-\bx\|^2~.
\end{equation}
It is easy to see that this assumption is necessary to have the convergence of the gradients to zero. Indeed, without smoothness the norm of the gradients does not go to zero even in the convex case, e.g., consider the function $f(x)=|x-1|$.

We assume that we have access to a stochastic first-order black-box oracle, that returns a noisy estimate of the gradient of $f$ at any point $\bx \in \R^d$. That is, we will use the following assumption:
\begin{itemize}
	\item[(\textbf{B1})] We receive a vector $g(\bx,\xi)$ such that $\E_\xi [g(\bx,\xi)] = \nabla f(\bx)$ for any $\bx \in \R^d$.
\end{itemize}

We will also make the following assumption on the variance of the noise.
\begin{itemize}
	\item[(\textbf{B2})](\textit{Sub-Gaussian Noise}) The stochastic gradient satisfies 
	$\E_\xi \left[\exp\left(\|\nabla f(\bx) - g(\bx,\xi)\|^2/\sigma^2\right)\right]\leq \exp(1), \ \forall \bx$.
\end{itemize}
The condition (\textbf{B2}) has been used by \citet{NemirovskiJLS09} and \citet{HarveyLPR19} to prove high probability convergence guarantees.
Intuitively, it implies that the tails of the noise distribution are dominated by tails of a Gaussian distribution. Note that, by Jensen's inequality, this condition implies a bounded variance on the stochastic gradients.

\section{A General Analysis for Algorithms with Momentum}
\label{sec:lemma}
In this section, we will consider a generic stochastic optimization algorithm with Polyak's momentum~\citep{Polyak64,Qian99,SutskeverMDH13}, also known as the Heavy-ball algorithm or classic momentum, see Algorithm~\ref{alg:momentum}.

\begin{algorithm}
	\caption{Algorithms with Momentum}
	\label{alg:momentum}
	\begin{algorithmic}[1]
		\STATE \textbf{Input:} $\bm_0 = 0$, $ \{\bta_t\}_{t=1}^T$, $0 < \mu \leq 1$, $\bx_1 \in \R^d$
		\FOR{$t = 1, \dots, T$}
		\STATE Get stochastic gradient $\bg_t = g(\bx_t, \xi_t)$
		\STATE $\bm_t = \mu \bm_{t-1} + \bta_t \bg_t$
		\STATE $\bx_{t+1} = \bx_t - \bm_t$
		\ENDFOR
	\end{algorithmic}
\end{algorithm}

\paragraph{Two forms of momentum, but not equivalent.}
First, we want to point out that there two forms of Heavyball algorithms are possible.
The first one is in Algorithm~\ref{alg:momentum}, while the second one is
\begin{equation}
\label{eq:momentum2}
\begin{split}
\bm_t &= \mu \bm_{t-1} + \bg_t, \\
\bx_{t+1} &= \bx_t - \bta_t \bm_t~.
\end{split}
\end{equation}
This second is used in many practical implementation, see, for example, PyTorch~\citep{PyTorch19}.
It would seem that there is no reason to prefer one over the other. However, here we argue that the classic form of momentum is the right one if we want to use adaptive learning rates. To see why, let's unroll the updates in both cases.
Using the update in Algorithm~\ref{alg:momentum}, we have
\[
\bx_{t+1} = \bx_t - \bta_t \bg_t - \mu \bta_{t-1} \bg_{t-1} - \mu^2 \bta_{t-1} \bg_{t-2} \dots,
\]
while using the update in \eqref{eq:momentum2}, we have
\[
\bx_{t+1} = \bx_t - \bta_t \bg_t - \mu \bta_{t} \bg_{t-1} - \mu^2 \bta_{t} \bg_{t-2} \dots~.
\]
In words, in the first case the update is composed by a sum of weighted gradients, each one multiplied by a learning rate we decided in the past. On the other hand, in the update \eqref{eq:momentum2} the update is composed by a sum of weighted gradients, each one multiplied by the \emph{current} learning rate. From the analysis point of view, the second update destroys the independence between the past and the future, introducing a dependency that breaks our analysis, unless we introduce very strict conditions on the gradients. On the other hand, the update in Algorithm~\ref{alg:momentum} allows us to carry out the analysis because each learning rate was chosen only with the knowledge of the past. Note that this is known problem in adaptive algorithms: the lack of independence between past and present is exactly the reason why Adam fails to converge on simple 1d convex problems, see for example the discussion in \citet{SavareseMBM19}.

It is interesting to note that usually people argue that these two types updates for momentum are usually considered equivalent. This seems indeed true only if the learning rates are not adaptive.

\paragraph{Assumptions on learning rates.} Note that in the pseudo-code we do not specify the learning rates $\bta_t \in \R^d$. In fact, our analysis covers the case of generic learning rates and adaptive ones too. We only need the following assumptions on the stepsizes $\bta_t$: 
\begin{itemize}
	\item[\textbf{(C1)}] $\bta_t$ is non-increasing, i.e., $\bta_{t+1} \leq \bta_t$, $\forall t$. 
	\item[\textbf{(C2)}] $\bta_t$ is independent with $\xi_t$.
\end{itemize}
The first assumption is very common \citep[e.g.,][]{DuchiHS11,ReddiKK18,LiO19,ChenLSH19,ZhouTYCG18}. Indeed, AdaGrad has the non-increasing step sizes by the definition. Also, \citet{ReddiKK18} have claimed that the main issue of the divergences of Adam and RMSProp lies in the positive definiteness of $1/\bta_t - 1/\bta_{t-1}$. 

The need of the second assumption is technical and shared by similar analysis \citep{LiO19, SavareseMBM19}. Indeed, \citet{LiO19} showed that delayed step sizes can avoid the possible deviation brought by the step sizes that include the current noise. 
\paragraph{High probability guarantee.} Adaptive learning rates and in general learning rates that are decided using previous gradients become stochastic variables. This makes the high probability analysis more complex. Hence, we use a new concentration inequality for martingales in which the variance is treated as a random variable, rather than a deterministic quantity. We use this concentration in the proof of Lemma~\ref{lemma:momentum}. Our proof, in the Appendix, merges ideas from the related results in \citet[Theorem 1]{BeygelzimerLLRS11} and \citet[Lemma 2]{LanNS12}. A similar result has also been shown by \citet[Lemma~6]{JinNGKJ19}. 
\begin{lemma}
	\label{lemma:sub_gaussian}
	Assume that $Z_1, Z_2, ..., Z_T$ is a martingale difference sequence with respect to $\xi_1, \xi_2, ..., \xi_T$ and $\E_t \left[\exp(Z_t^2/\sigma_t^2)\right] \leq \exp(1)$ for all $1\leq t \leq T$, where $\sigma_t$ is a sequence of random variables with respect to $\xi_1, \xi_2, \dots, \xi_{t-1}$. 
	Then, for any fixed $\lambda > 0$ and $\delta \in (0,1)$, with probability at least $1-\delta$, we have 
	\[
	\sum_{t=1}^T Z_t \leq \frac{3}{4} \lambda \sum_{t=1}^T \sigma_t^2 + \frac{1}{\lambda} \ln \frac{1}{\delta}~.
	\]
\end{lemma}
\paragraph{Main result.} We can now present our general lemma, that allows to analyze SGD with momentum with adaptive learning rates. We will then instantiate it for particular examples.
\begin{lemma}
	\label{lemma:momentum}
	Assume (\textbf{A}, \textbf{B1}, \textbf{B2}, \textbf{C1}, \textbf{C2}). Then, for any $\delta \in (0,1)$, with probability at least $1-\delta$, the iterates of Algorithm \ref{alg:momentum} satisfy
	\begin{align*}
	\sum_{t=1}^{T} &\langle \bta_t, \nabla f(\bx_t) ^2 \rangle 
	\leq \frac{3 \|\bta_1\|\sigma^2(1- \mu^T)^2 }{(1- \mu)^2}\ln \frac{1}{\delta} \\
	&\quad + 2(f(\bx_1) - f^{\star}) 
	+ \frac{M(3-\mu)}{1- \mu}\sum_{t=1}^{T}  \| \bta_t\bg_t \|^2
	~. 
	\end{align*}
\end{lemma}
Lemma~\ref{lemma:momentum} accomplishes the task of upper bounding the inner product $\sum_{t=1}^{T}\langle \bta_t \nabla f(\bx_t), \bm_t \rangle$. Then, it is easy to lower bound the l.h.s by $\sum_{t=1}^{T} \langle \bta_T, \nabla f(\bx_t)^2 \rangle$ using the assumption \textbf{(C1)}, followed by the upper bound of $\sum_{t=1}^{T} \| \bta_t\bg_t \|^2$ based on the setting of $\bta_t$.

\subsection{SGD with Momentum with $\frac{1}{\sqrt{t}}$ Learning Rates}
\label{ssec:sqrt}
As a warm-up, we now use Lemma~\ref{lemma:momentum} to prove a high probability convergence guarantee for the simple case of deterministic learning rates of $\eta_{t,i}=\frac{c}{\sqrt{t}}$.
\begin{theorem}
	\label{thm:momentum_sqrt}
	Let $T$ the number of iterations of Algorithm~\ref{alg:momentum}.
	Assume (\textbf{A}, \textbf{B1}, \textbf{B2}). Set step size $\bta_t$ as $ \eta_{t,i} = \frac{c}{\sqrt{t}}, i= 1, \cdots, d$, where $ c \leq \frac{1- \mu^T}{4M(3-2\mu)}$. Then, for any $\delta \in (0,1)$, with probability at least $1-\delta$, the iterates of Algorithm \ref{alg:momentum} satisfy
	\begin{align*}
	\min_{1\leq t \leq T} \| \nabla f(\bx_t) \|^2
	& \leq \frac{4(f(\bx_1) - f^{\star})}{c \sqrt{T}} +  \frac{6(1- \mu^T)^2\sigma^2}{(1-\mu)^2 \sqrt{T}} \\
	& \quad + \frac{4(3-\mu)cM\sigma^2 \ln \frac{2Te}{\delta} \ln T }{(1- \mu)\sqrt{T}} ~.
	\end{align*}
\end{theorem}
The proof is in Appendix.
\subsection{AdaGrad with Momentum}
\label{sec:adagrad}
Now, we are going to prove the convergence rate of a variant AdaGrad in which we use momentum and learning rates that do not contain the current gradient. That is, the step sizes are defined as $\bta_t = (\eta_{t,j})_{j=1,\dots,d}$
\begin{equation}
\label{eq:bta}
\eta_{t,j}=\frac{\alpha}{\sqrt{\beta+ \sum_{i=1}^{t-1} g_{i,j}^2}}, \ j=1, \cdots,d,
\end{equation}
where $\alpha, \beta >0$. Removing the current gradient from the learning rate was proposed in \citet{LiO19, SavareseMBM19}. Following the naming style in \cite{SavareseMBM19}, we denote this variant by Delayed AdaGrad. 

Obviously, \eqref{eq:bta} satisfies \textbf{(C1)} and \textbf{(C2)}. Hence, we are able to employ Lemma~\ref{lemma:momentum} to analyze this variant. Moreover, for Delayed AdaGrad, we upper bound $ \sum_{t=1}^T \| \bta_t \bg_t \| ^2$ with the following lemma, whose proof is in the Appendix.
\begin{lemma}
	\label{lemma:bound_1}
	Assume (\textbf{A}, \textbf{B1}, \textbf{B2}). Let $\bta_t$ set as in \eqref{eq:bta}, where $\alpha, \beta > 0$. Then, for any $\delta \in (0,1)$, with probability at least $1-\delta$, we have 
	\begin{align*}
	& \sum_{t=1}^T \| \bta_t \bg_t \| ^2  
	\leq \frac{4d\alpha^2\sigma^2}{\beta} \ln \frac{2Te}{\delta} + \frac{4\alpha }{\sqrt{\beta}} \sum_{t=1}^{T}  \langle \bta_t, \nabla f(\bx_t)^2\rangle\\
	& \leq 2\alpha^2d \ln \left( \sqrt{ \beta + \frac{2T\sigma^2 \ln \tfrac{2Te}{\delta}}{d} } + \sqrt{\frac{2}{d}\sum_{t=1}^T \|\nabla f(\bx_t) \|^2} \right).
	\end{align*}
\end{lemma}
We now present the convergence guarantee for Delayed AdaGrad with momentum.
\begin{theorem}[Delayed AdaGrad with Momentum]
	\label{thm:momentum_adagrad}
	Assume (\textbf{A}, \textbf{B1}, \textbf{B2}). Let $\bta_t$ set as in \eqref{eq:bta}, where $\alpha, \beta > 0$ and $4 \alpha \leq  \frac{\sqrt{\beta}(1-\mu)^2}{2M(1+\mu)}$. Then, for any $\delta
	\in (0,1)$, with probability at least $1- \delta$, the iterates of Algorithm \ref{alg:momentum} satisfy
	\begin{align*}
	&\min_{1\leq t \leq T} \| \nabla f(\bx_t) \|^2 \\
	&\leq \frac{1}{T} \max \left( \frac{4 C(T)^2}{ \alpha^2}, \frac{C(T)}{\alpha} \sqrt{2\beta + 4T\sigma^2 \ln \frac{3Te}{\delta}} \right),
	\end{align*}
	where 
	$C(T) = O\left(\frac{1}{\alpha} + \frac{d \left(\alpha+ \sigma^2 \left(\alpha \ln \frac{T}{\delta} + \frac{\ln \frac{1}{\delta}}{1- \mu}\right)\right)}{1- \mu}\right)$.
\end{theorem}
\paragraph{Adaptivity to Noise.} Observe that when $\sigma = 0$, the convergence rate recovers the rate of Gradient Descent if $O(\frac{1}{T})$ with a constant learning rate. On the other hand, in the noisy case, it matches the rate of SGD $O(\frac{\sigma}{\sqrt{T}})$ with the optimal worst-case learning rate of $O(\frac{1}{\sigma \sqrt{t}})$. In other words, with a unique learning rate, we recover two different optimal convergence rates that requires two different learning rates and the knowledge of $\sigma$.
This adaptivity of Delayed AdaGrad was already proved in \citet{LiO19}, but only in expectation and without a momentum term.

\paragraph{Dependency on $\mu$.} Observe that the convergence upper bound increases over $\mu \in (0, 1)$ and the optimal upper bound is achieved when taking the momentum parameter $\mu = 0$. In words, the algorithms without momentums have the best theoretical results. This is a known caveat for this kind of analysis and a similar behavior w.r.t. $\mu$ is present, e.g., in \citet[Theorem~1]{ZouSJSL18} for algorithms with Polyak's momentum.

\section{Conclusion and Future Work.}
In this work, we present a high probability analysis of adaptive SGD with Polyak's momentum in the nonconvex setting. Without using the common assumption of bounded gradients nor bounded noise, we give the high probability bound for SGD with Polyak's momentum with step size $O(\frac{1}{\sqrt{t}})$ and for Delayed AdaGrad with momentum. In particular, to the best of our knowledge, this is the \textit{first} high probability convergence guarantee for adaptive methods. 

In the future, we plan to extend our results to more adaptive methods, such as Adam and AMSGrad~\citep{TielemanH12}. Moreover, we will explore other forms of momentum such as exponential moving average and Nesterov's momentum~\citep{Nesterov83}.

\section*{Acknowledgements}
This material is based upon work supported by the National Science Foundation under grant no. 1925930 ``Collaborative Research: TRIPODS Institute for Optimization and Learning''.
	
\balance
\bibliography{learning}
\bibliographystyle{icml2020}

\newpage
\onecolumn
\appendix
\section{Appendix}
\subsection{Details of Section~\ref{sec:lemma}}
\begin{proof}[Proof of Lemma~\ref{lemma:sub_gaussian}]
	Set $\tilde{Z_t} = Z_t / \sigma_t$. By the assumptions of $Z_t$ and $\sigma_t$, we have
	\[
	\E_t [\tilde{Z_t}] = \frac{1}{\sigma_t} \E_t [Z_t] = 0 
	\quad
	\text{ and }
	\quad 
	\E_t \left[ \exp \left( \tilde{Z}_t^2 \right) \right] \leq \exp (1)~.
	\]
	By Jensen's inequality, it follows that for any $c \in [0,1]$, 
	\begin{equation}
	\label{eq: sub_gaussian_coef}
	\E_t \left[ \exp \left(c \tilde{Z}_t^2 \right) \right]
	= \E_t \left[ \left( \exp \left(\tilde{Z}_t^2\right) \right)^c \right]
	\leq \left( \E_t \left[ \exp \left(\tilde{Z}_t^2\right)\right] \right)^c \leq \exp(c)~. 
	\end{equation}
	Also it can be verified that $\exp (x) \leq x + \exp (9x^2/16)$ for all $x$, hence for $ |\kappa| \in [0, 4/3]$ we get 
	\begin{equation}
	\label{eq: sub_gaussian1}
	\E_t \left[ \exp \left( \kappa \tilde{Z}_t \right) \right] 
	\leq \E_t \left[ \exp \left(9 \kappa^2 \tilde{Z}_t^2/16 \right) \right] 
	\leq \exp \left( 9\kappa^2/16\right)  
	\leq \exp \left( 3 \kappa^2/4\right),
	\end{equation}
	where in the second inequality, we used~\eqref{eq: sub_gaussian_coef}.
	Besides, $k x \leq 3k^2/8 + 2x^2/3 $ holds for any $k$ and $x$. Hence for $ |\kappa| \geq 4/3$, we get
	\begin{equation}
	\label{eq: sub_gaussian2}
	\E_t  \left[ \exp \left( \kappa \tilde{Z}_t \right) \right] 
	\leq \exp \left( 3\kappa^2/8 \right) \E_t \left[\exp \left(2\tilde{Z}_t^2/3\right) \right]
	\leq \exp \left( 3\kappa^2/8 + 2/3 \right) \leq \exp \left(3\kappa^2/4  \right),
	\end{equation}
	where in the second inequality we used~\eqref{eq: sub_gaussian_coef}. 
	Combining \eqref{eq: sub_gaussian1} and \eqref{eq: sub_gaussian2}, we get $\forall \kappa$, 
	\begin{equation}
	\label{eq:sub_gaussian3}
	\E_t \left[ \exp \left( \kappa \tilde{Z}_t \right) \right] \leq \exp \left( 3\kappa^2/4  \right)~. 
	\end{equation}
	Note that the above analysis for~\eqref{eq:sub_gaussian3} still hold when $\kappa$ is a random variable with respect to $\xi_1,\xi_2,\dots, \xi_{t-1}$.
	So for $Z_t$, we have $\E_t \left[ \exp \left(\lambda Z_t \right) \right] \leq \exp \left( 3\lambda^2 \sigma_t^2/4 \right), \quad \lambda>0$.
	
	Define the random variables $Y_0 = 1$ and $Y_t = Y_{t-1} \exp \left(\lambda Z_t - 3\lambda^2 \sigma_t^2/4  \right), \quad 1\leq t\leq T$.
	So, we have $E_t Y_t = Y_{t-1} \exp \left( -3\lambda^2 \sigma_t^2/4  \right) \cdot \E_t \left[ \exp \left( \lambda Z_t \right) \right] \leq Y_{t-1}$.
	Now, taking full expectation over all variables $\xi_1, \xi_2, \dots ,\xi_T$, we have 
	\[
	\E Y_T \leq \E Y_{T-1} \leq \dots \leq \E Y_0 = 1~. 
	\]
	By Markov's inequality, $P \left( Y_T \geq \frac{1}{\delta} \right) \leq \delta$, and 
	$
	Y_T = \exp \left( \lambda \sum_{t=1}^T Z_t- \frac{3}{4} \lambda^2 \sum_{t=1}^T \sigma_t^2 \right), 
	$
	we have 
	\begin{align*}
	P \left( Y_T \geq \frac{1}{\delta} \right) 
	= P \left( \lambda \sum_{t=1}^T Z_t- \frac{3}{4} \lambda^2 \sum_{t=1}^T \sigma_t^2  \geq \ln \frac{1}{\delta} \right) 
	= P \left( \sum_{t=1}^T Z_t \geq  \frac{3}{4} \lambda \sum_{t=1}^T \sigma_t^2 + \frac{1}{\lambda} \ln \frac{1}{\delta} \right),
	\leq \delta,
	\end{align*}
	which completes the proof. 
\end{proof}

To prove Lemma~\ref{lemma:momentum}, we first need the following technical Lemma.
\begin{lemma}
	\label{lemma:doublesum}
	$\forall T \geq 1$, it holds 
	\[
	\sum_{t=1}^{T} a_t \sum_{i=1}^{t} b_i  = \sum_{t=1}^{T} b_t \sum_{i=t}^{T} a_i
	\quad
	\text{and}
	\quad
	\sum_{t=1}^{T} a_t \sum_{i=0}^{t-1} b_i  = \sum_{t=0}^{T-1} b_t \sum_{i=t+1}^{T}a_i~.
	\]
\end{lemma}
%
\begin{proof}
	We prove these equalities by induction. 
	When $T=1$, they obviously hold.
	Now, for $k < T$, assume that $\sum_{t=1}^{k} a_t \sum_{i=1}^{t} b_i  = \sum_{t=1}^{k} b_t \sum_{i=t}^{k} a_i$. 
	Then, we have
	\begin{align*}
	\sum_{t=1}^{k+1} a_t \sum_{i=1}^{t} b_i  
	& = \sum_{t=1}^{k} a_t \sum_{i=1}^{t} b_i  + a_{k+1} \sum_{i=1}^{k+1} b_i 
	=  \sum_{t=1}^{k} b_t \sum_{i=t}^{k} a_i + a_{k+1} \sum_{i=1}^{k} b_i + a_{k+1}b_{k+1}\\
	& = \sum_{t=1}^{k} b_t \sum_{i=t}^{k+1} a_i + a_{k+1}b_{k+1}
	=  \sum_{t=1}^{k+1} b_t \sum_{i=t}^{k+1} a_i~.
	\end{align*}
	Hence, by induction, the equality is proved.
	
	Similarly, for second equality assume that for $k < T$ we have
	$\sum_{t=1}^{k} a_t \sum_{i=0}^{t-1} b_i  = \sum_{t=0}^{k-1} b_t \sum_{i=t+1}^{k}a_i$.
	Then, we have
	\begin{align*}
	\sum_{t=1}^{k+1} a_t \sum_{i=0}^{t-1} b_i  
	& = \sum_{t=1}^{k} a_t \sum_{i=0}^{t-1} b_i  + a_{k+1} \sum_{i=0}^{k} b_i 
	=  \sum_{t=0}^{k-1} b_t \sum_{i=t+1}^{k}a_i+ a_{k+1} \sum_{i=0}^{k-1} b_i + a_{k+1}b_{k}\\
	& = \sum_{t=0}^{k-1} b_t \sum_{i=t+1}^{k+1} a_i + a_{k+1}b_{k}
	=  \sum_{t=0}^{k} b_t \sum_{i=t+1}^{k+1} a_i~. 
	\end{align*}
	By induction, we finish the proof.  
\end{proof}

\begin{proof}[Proof of Lemma~\ref{lemma:momentum}]
	By the smoothness of $f$ and the definition of $\bx_{t+1}$, we have 
	\begin{equation}
	\label{eq:adamsmooth-coodinate}
	\begin{aligned}
	f(\bx_{t+1})-f(\bx_t) 
	\leq -\langle \nabla f(\bx_t), \bm_t \rangle + \frac{M}{2} \| \bm_t \|^2~. 
	\end{aligned}
	\end{equation}
	We now upper bound $-\langle \nabla f(\bx_t), \bm_t \rangle$. 
	\begin{align*}
	& -\langle \nabla f(\bx_t), \bm_t \rangle \\
	& = -\mu \langle \nabla f(\bx_t), \bm_{t-1}\rangle - \langle \nabla f(\bx_t), \bta_t \bg_t \rangle \\
	& = - \mu \langle \nabla f(\bx_{t-1}), \bm_{t-1}\rangle 
	- \mu \langle \nabla f(\bx_t) - \nabla f(\bx_{t-1}), \bm_{t-1}\rangle - \langle \nabla f(\bx_t), \bta_t \bg_t \rangle \\
	& \leq - \mu \langle \nabla f(\bx_{t-1}), \bm_{t-1}\rangle + \mu \| \nabla f(\bx_t) - \nabla f(\bx_{t-1}) \| \| \bm_{t-1}\| - \langle \nabla f(\bx_t), \bta_t\bg_t \rangle\\
	& \leq - \mu \langle \nabla f(\bx_{t-1}),\bm_{t-1}\rangle + \mu M \|\bm_{t-1}\|^2 -\langle \nabla f(\bx_t), \bta_t \bg_t \rangle, 
	\end{align*}
	where the second inequality is due to the smoothness of $f$.
	Hence, iterating the inequality we have
	\begin{align*}
	-\langle \nabla f(\bx_t),\bta_t \bm_t \rangle 
	& \leq - \mu^2 \langle \nabla f(\bx_{t-2}), \bm_{t-2} \rangle + \mu^2 M \|\bm_{t-2}\|^2 + \mu M \|\bm_{t-1}\|^2\\
	& \quad - \mu \langle \nabla f(\bx_{t-1}), \bta_{t-1} \bg_{t-1} \rangle - \langle \nabla f(\bx_t), \bta_t \bg_t \rangle \\
	& \leq M\sum_{i=1}^{t-1} \mu^{t-i} \|\bm_{i}\|^2 - \sum_{i=1}^t \mu^{t-i} \langle \nabla f(\bx_i), \bta_i \bg_i\rangle~. \qedhere
	\end{align*}
	Thus, denoting by $\bps_t = \bg_t - \nabla f(\bx_t)$ and summing \eqref{eq:adamsmooth-coodinate} over $t$ from $1$ to $T$, we obtain 
	\begin{align*}
	f^\star-f(\bx_1) 
	& \leq f(\bx_{T+1})-f(\bx_1) \\
	& \leq  M \sum_{t=1}^{T}\sum_{i=1}^{t-1} \mu^{t-i} \|\bm_{i} \|^2
	- \sum_{t=1}^{T} \sum_{i=1}^t \mu^{t-i}\langle \nabla f(\bx_i),  \bta_i \bg_i\rangle +  \sum_{t=1}^{T} \frac{M}{2} \| \bm_t \|^2\\	
	& \leq  M \sum_{t=1}^{T}\sum_{i=1}^{t-1} \mu^{t-i} \|\bm_{i} \|^2
	- \sum_{t=1}^{T} \langle \bta_t, \nabla f(\bx_t) ^2 \rangle 
	- \sum_{t=1}^{T} \sum_{i=1}^t \mu^{t-i}\langle \nabla f(\bx_i),  \bta_i \bps_i\rangle \\
	& \quad + \sum_{t=1}^{T} \frac{M}{2} \| \bm_t \|^2~.
	\end{align*}
	By Lemma \ref{lemma:doublesum}, we have
	\begin{align*}
	M \sum_{t=1}^{T}\sum_{i=1}^{t-1}  \mu^{t-i}\| \bm_{i}\|^2
	\leq \frac{M}{1-\mu} \sum_{t=1}^{T} \| \bm_{t}\|^2~.
	\end{align*}
	Also, by Lemma~\ref{lemma:doublesum}, we have  
	\begin{align*}
	- \sum_{t=1}^{T} \sum_{i=1}^t \mu^{t-i}\langle \nabla f(\bx_i), \bta_i \bps_i\rangle
	&= - \sum_{t=1}^{T} \mu^{-t} \langle \nabla f(\bx_t), \bta_t \bps_t\rangle \sum_{i=t}^T \mu^{i} \\
	& = - \frac{1}{1-\mu}\sum_{t=1}^{T} \langle \nabla f(\bx_t), \bta_t  \bps_t\rangle (1- \mu^{T-t+1})
	\triangleq S_T~.
	\end{align*}
	We then upper bound $S_T$. 
	Denote by $L_t:= - \frac{1- \mu^{T-t+1}}{1- \mu}\langle \nabla f(\bx_t), \bta_t \bps_t \rangle$, and $N_t := \frac{(1- \mu^{T-t+1})^2}{(1- \mu)^2}\| \bta_t\nabla f(\bx_t)\|^2 \sigma^2$. Using the assumptions on the noise, for any $1 \leq t \leq T$, we have
	\[
	\exp\left(\frac{L_t^2 }{ N_t }\right) 
	\leq \exp \left(\frac{\| \bta_t \nabla f(\bx_t) \|^2 \| \bps_t \|^2 (1- \mu^{T-t+1})^2/(1- \mu)^2 }{N_t}\right)
	= \exp \left(\frac{\| \bps_t\|^2 }{ \sigma^2 }\right) 
	\leq \exp(1)~. 
	\]
	We can also see that for any $t$, $\E_t [ L_t]  = - \frac{1- \mu^{T-t+1}}{1- \mu}\sum_{i=1}^{d} \bta_{t,i} \nabla f(\bx_t)_i \E_t [\bps_{t,i}] = 0$.
	Thus, from Lemma~\ref{lemma:sub_gaussian}, with probability at least $1-\delta$, any $\lambda >0$, we have 
	\begin{align*}
	S_T = \sum_{t=1}^T L_t \leq \frac{3}{4}\lambda \sum_{t=1}^T N_t + \frac{1}{\lambda} \ln \frac{1}{\delta}
	& \leq\frac{3\lambda(1- \mu^T)^2}{4(1- \mu)^2}\sum_{t=1}^T \| \bta_t\nabla f(\bx_t)\|^2 \sigma^2  + \frac{1}{\lambda} \ln \frac{1}{\delta} \\
	&\leq \frac{3\lambda \| \bta_1\|(1- \mu^T)^2}{4(1- \mu)^2}\sum_{t=1}^T \langle \bta_t, \nabla f(\bx_t)^2 \rangle \sigma^2  + \frac{1}{\lambda} \ln \frac{1}{\delta} ~.\qedhere
	\end{align*}
	Finally, we upper bound $\sum_{t=1}^{T} \| \bm_t \|^2$. From the convexity of $\| \cdot \|^2$, we have
	\begin{align*}
	\| \bm_t \|^2 = \left\| \mu \bm_{t-1} + (1- \mu) \frac{\bta_t \bg_t}{1-\mu} \right\|^2 
	\leq \mu \| \bm_{t-1}\|^2 + \frac{1}{1- \mu} \| \bta_t \bg_t \|^2~. 
	\end{align*}
	Summing over $t$ from $1$ to $T$, we have 
	\begin{align*}
	\sum_{t=1}^T \| \bm_t \|^2
	& \leq \sum_{t=1}^T \mu \|  \bm_{t-1}\|^2 + \frac{1}{1-\mu}\sum_{t=1}^T \| \bta_t \bg_t \|^2 \\
	& =  \sum_{t=1}^{T-1} \mu \| \bm_t\|^2 + \frac{1}{1-\mu}\sum_{t=1}^T \| \bta_t \bg_t \|^2 \\
	& \leq  \sum_{t=1}^T \mu \| \bm_t\|^2 + \frac{1}{1-\mu}\sum_{t=1}^T \| \bta_t \bg_t \|^2 ,
	\end{align*}
	where in the first equality we used $\bm_0=0$. Reordering the terms, we have that
	\[
	\sum_{t=1}^T \| \bm_t \|^2 \leq \frac{1}{(1-\mu)^2}\sum_{t=1}^T \| \bta_t \bg_t \|^2~.
	\]
	Combining things together, and taking $\lambda = \frac{2(1- \mu)^2}{3\| \bta_1 \| (1- \mu^T)^2 \sigma^2} $, with probability at least $1- \delta$, we have 
	\begin{align*}
	f^\star-f(\bx_1) 
	& \leq \frac{1}{\lambda} \ln \frac{1}{\delta}
	+ \sum_{t=1}^{T}\left[\left( \frac{M}{2} + \frac{M }{1- \mu }\right) \| \bta_t \bg_t \|^2
	-  \left(1 - \frac{3\lambda \| \bta_1 \| (1- \mu^T)^2 \sigma^2 }{4(1-\mu)^2} \right) \langle \bta_t ,\nabla f(\bx_t)^2 \rangle\right]\\
	&= \frac{3 \|\bta_1\|(1-\mu^T)^2 \sigma^2}{2(1-\mu)^2} \ln \frac{1}{\delta} 
	+ \sum_{t=1}^{T}\left[\frac{(3-\mu)M}{2(1- \mu)}\| \bta_t \bg_t \|^2
	-  \frac{1}{2}\langle \bta_t ,\nabla f(\bx_t)^2 \rangle\right]~.
	\end{align*}
	Rearranging the terms, we get the stated bound. 
\end{proof}

\subsection{Details of Section~\ref{ssec:sqrt}}
The proof of this Theorem~\ref{thm:momentum_sqrt} makes use of the following additional Lemma on the tail of sub-gaussian noise.
\begin{lemma}
	\label{lemma:max_bound}
	Assume \textbf{B2}, then for any $\delta \in (0,1)$, with probability at least $1-\delta$, we have 
	\[
	\max_{1 \leq t \leq T} \| \bg_t - \nabla f(\bx_t) \|^2 \leq \sigma^2 \ln \frac{Te}{\delta}~. 
	\]
\end{lemma}
\begin{proof}
	By Markov's inequality, for any $A>0$, 
	\begin{align*}
	\Pr \left(\max_{1\leq t \leq T} \|\bg_t- \nabla f(\bx_t)\|^2 > A\right)
	& = \Pr \left(\exp \left( \frac{\max_{1\leq t \leq T} \|\bg_t- \nabla f(\bx_t)\|^2 }{\sigma^2} \right) > \exp \left( \frac{A}{\sigma^2}\right) \right)\\
	& \leq \exp \left(-\frac{A}{\sigma^2} \right)\E \left[\exp \left(\frac{\max_{1\leq t \leq T} \|\bg_t- \nabla f(\bx_t)\|^2 }{\sigma^2}\right) \right]\\
	& =  \exp \left(-\frac{A}{\sigma^2} \right) \E \left[ \max_{1\leq t \leq T} \exp \left(\frac{\|\bg_t - \nabla f(\bx_t)\|^2 }{\sigma^2}\right) \right]\\
	& \leq \exp \left(-\frac{A}{\sigma^2} \right) \sum_{t=1}^T \E \left[ \exp \left( \frac{\| \nabla f(\bx_t)- \bg_t \|^2 }{\sigma^2} \right) \right] 
	\leq \exp \left(-\frac{A}{\sigma^2}+1 \right) T~. \qedhere
	\end{align*}
\end{proof}

\begin{proof}[Proof of Theorem~\ref{thm:momentum_sqrt}]
	With the fact that $\| \ba+\bb\|^2 \leq 2\| \ba\|^2 + 2\| \bb\|^2$, we have 
	\begin{align*}
	\sum_{t=1}^{T} \| \bta_t \bg_t \|^2 = 
	\sum_{t=1}^{T} \eta_t^2 \| \bg_t \|^2
	& \leq \sum_{t=1}^{T} 2\eta_t^2 \| \nabla f(\bx_t )\|^2+  \sum_{t=1}^{T} 2 \eta_t^2 \| \bg_t - \nabla f(\bx_t) \|^2\\
	& \leq \sum_{t=1}^{T} 2\eta_t^2 \| \nabla f(\bx_t )\|^2+ \max_{1\leq t \leq T} \| \bg_t - \nabla f(\bx_t) \|^2\sum_{t=1}^{T} 2 \eta_t^2~.
	\end{align*}
	By Lemma~\ref{lemma:max_bound}, Lemma~\ref{lemma:momentum} and the union bound, we have that with probability at least $1-\delta$, 
	\begin{align*}
	\frac{\eta_T}{2} \sum_{t=1}^{T} \| \nabla f(\bx_t) \|^2 
	& \leq \left(1- \frac{2M(3- \mu)}{1- \mu}\eta_1\right) \sum_{t=1}^{T} \eta_t \| \nabla f(\bx_t) \|^2 \\
	& \leq 2 (f(\bx_1) - f^{\star}) +  \frac{2(3-\mu)c^2M\sigma^2 \ln \frac{2Te}{\delta} \ln T }{1- \mu} + \frac{3c(1- \mu^T)^2 \sigma^2}{(1- \mu)^2} \ln \frac{1}{\delta}~. 
	\end{align*}
	Rearranging the terms and lower bounding  $\sum_{t=1}^{T} \| \nabla f(\bx_t) \|^2$ by $T \cdot \min_{1\leq t \leq T} \| \nabla f(\bx_t) \|^2$, we have the stated bound. 
\end{proof}

\subsection{Details of Section~\ref{sec:adagrad}}

For the proof of Lemma~\ref{lemma:bound_1}, we first need the following technical Lemma.
\begin{lemma}
	\label{lemma:sum_integral_bounds}
	Let $a_i\geq0, \cdots, T$ and $f:[0,+\infty)\rightarrow [0, +\infty)$ non-increasing function.
	Then
	\begin{align*}
	\sum_{t=1}^T a_t f\left(a_0+\sum_{i=1}^{t} a_i\right) 
	&\leq \int_{a_0}^{\sum_{t=0}^T a_t} f(x) dx~.
	\end{align*}
\end{lemma}
\begin{proof}
	Denote by $s_t=\sum_{i=0}^{t} a_i$. Then, we have
	\begin{align*}
	a_i f(s_i) 
	=  \int_{s_{i-1}}^{s_i} f(s_i) d x 
	\leq \int_{s_{i-1}}^{s_i} f(x) d x~.
	\end{align*}
	Summing over $i=1, \cdots, T$, we have the stated bound.
\end{proof}

\begin{proof}[Proof of Lemma~\ref{lemma:bound_1}]
	First, we separate $\sum_{t=1}^T \|\bta_t\bg_t\|^2 $ into two terms:
	\[
	\sum_{t=1}^T \|\bta_t\bg_t\|^2 
	= \sum_{t=1}^T \|\bta_{t+1}\bg_t\|^2 + \sum_{t=1}^T \langle \bta_t^2 - \bta_{t+1}^2, \bg_t^2\rangle~.
	\]
	Then, we proceed 
	\begin{align}
	\sum_{t=1}^T \langle \bta_t^2 - \bta_{t+1}^2, \bg_t^2\rangle
	& = \sum_{i=1}^{d}\sum_{t=1}^T (\eta_{t,i}^2 - \eta_{t+1, i}^2) g_{t,i}^2 \nonumber\\
	& \leq \sum_{i=1}^{d}\sum_{t=1}^T2\eta_{t,i} g_{t,i}^2 (\eta_{t,i}-\eta_{t+1,i})\nonumber\\
	& \leq  2\sum_{i=1}^{d} \max_{1\leq t \leq T} \eta_{t,i} g_{t,i}^2 \sum_{t=1}^T(\eta_{t,i}-\eta_{t+1,i})\nonumber\\
	& \leq 2\sum_{i=1}^{d} \eta_{1,i} \max_{1\leq t \leq T} \eta_{t,i} g_{t,i}^2 \nonumber\\
	& \leq 4\sum_{i=1}^{d} \eta_{1,i} \max_{1\leq t \leq T} \eta_{t,i}\left(g_{t,i}^2 - \nabla f(\bx_t)_i^2\right) + 4\sum_{i=1}^{d} \eta_{1,i}  \sum_{t=1}^{T}\eta_{t,i}\nabla f(\bx_t)_i^2\nonumber\\
	& \leq 4\sum_{i=1}^{d} \eta_{1,i}^2 \max_{1\leq t \leq T} |g_{t,i}^2 - \nabla f(\bx_t)_i^2| + 4\sum_{i=1}^{d} \eta_{1,i}  \sum_{t=1}^{T}\eta_{t,i}\nabla f(\bx_t)_i^2\nonumber\\
	& \leq \frac{4d \alpha^2}{\beta} \max_{1\leq t \leq T} \| \bg_t - \nabla f(\bx_t) \|^2 + \frac{4\alpha}{\sqrt{\beta}} \sum_{t=1}^{T} \langle \bta_t, \nabla f(\bx_t)^2\rangle~. \label{eq:uni_case}
	\end{align}
	Using Lemma~\ref{lemma:max_bound} on \eqref{eq:uni_case}, for $\delta \in (0,1)$, with probability at least $1- \frac{\delta}{2}$, we have 
	\begin{equation}
	\sum_{t=1}^T \langle \bta_t^2 - \bta_{t+1}^2, \bg_t^2\rangle
	\leq \frac{4d \alpha^2 \sigma^2}{\beta} \ln \frac{2Te}{\delta} + \frac{4\alpha}{\sqrt{\beta}} \sum_{t=1}^{T} \langle \bta_t, \nabla f(\bx_t)^2\rangle~.
	\end{equation}
	We now upper bound $\sum_{t=1}^T \|\bta_{t+1}\bg_t\|^2$:
	\begin{align}
	\sum_{t=1}^T \|\bta_{t+1} \bg_t \|^2 
	& = \sum_{i=1}^{d} \sum_{t=1}^T \frac{\alpha^2 g_{t,i}^2}{\beta + \sum_{j=1}^{t} g_{j,i}^2} \nonumber\\
	& \leq \sum_{i=1}^{d} \alpha^2 \ln \left(\beta + \sum_{t=1}^{T} g_{t,i}^2 \right) \nonumber\\
	& \leq \alpha^2d \ln \left(\beta +  \frac{1}{d}\sum_{i=1}^{d}\sum_{t=1}^{T} g_{t,i}^2 \right) \nonumber\\
	& =  2\alpha^2d \ln \left( \sqrt {\beta +  \frac{1}{d} \sum_{t=1}^{T} \| \bg_t \|^2} \right) \nonumber\\
	& \leq 2\alpha^2d \ln \left( \sqrt{ \beta + \frac{2T}{d} \max_{1 \leq t \leq T} \|\bg_t - \nabla f(\bx_t) \|^2} + \sqrt{ \frac{2}{d}\sum_{t=1}^T \|\nabla f(\bx_t) \|^2} \right), \label{eq:eps=0}
	\end{align}
	where in the first inequality we used Lemma~\ref{lemma:sum_integral_bounds} and in the second inequality we used Jensen's inequality. Then using Lemma~\ref{lemma:max_bound} on \eqref{eq:eps=0}, with probability at least $1- \frac{\delta}{2}$, we have 
	\[
	\sum_{t=1}^T \|\bta_{t+1} \bg_t \|^2 
	\leq 2\alpha^2d \ln \left( \sqrt{ \beta + \frac{2T \sigma^2}{d} \ln \frac{2Te}{\delta}} + \sqrt{ \frac{2}{d}\sum_{t=1}^T \|\nabla f(\bx_t) \|^2} \right)~.
	\]
	Putting things together, we have the stated bound. 
\end{proof}

Finally, to prove Theorem~\ref{thm:momentum_adagrad}, we need the two following Lemmas.
\begin{lemma}[Lemma~6 in \citet{LiO19}]
	\label{lemma:logsolvex}
	Let $x \geq 0$, $A, C, D \geq 0$, $B>0$, and $x^2 \leq (A+Bx)(C+D\ln (A+Bx))$. Then.
	\[
	x < 32 B^3 D^2 + 2 B C + 8 B^2 D \sqrt{C} + A/B~.
	\]
\end{lemma}
\begin{lemma}[Lemma~5 in \citet{LiO19}]
	\label{lemma:solve_x}
	If $x \geq 0$, and $x \leq C \sqrt{A+Bx}$, then $x \leq \max \left( 2BC^2, C\sqrt{2A}\right)$. 
\end{lemma}

\begin{proof}[Proof of Theorem~\ref{thm:momentum_adagrad}]
	By Lemma~\ref{lemma:momentum} and Lemma~\ref{lemma:bound_1}, for $\delta \in (0,1)$, with probability at least $1 - \frac{2}{3} \delta$, we have 
	\begin{align*}
	& \left(1- \frac{4\alpha M (3- \mu)}{\sqrt{\beta} (1- \mu)}\right) \sum_{t=1}^{T} \langle \bta_t, \nabla f(\bx_t)^2 \rangle \\
	& \leq 2(f(\bx_1) - f^{\star}) + \frac{M(3-\mu)}{1- \mu}  \left( K	+ \frac{4d\alpha^2\sigma^2}{\beta} \ln \frac{3Te}{\delta} \right) + \frac{3\| \bta_1 \| \sigma^2 (1- \mu^T)^2}{(1- \mu)^2} \ln \frac{3}{\delta}~. 
	\end{align*}
	where $K$ denotes $2\alpha^2d \ln \left( \sqrt{ \beta + \frac{2T\sigma^2}{d} \ln \frac{2Te}{\delta}} + \sqrt{\frac{2}{d}} \sqrt{\sum_{t=1}^T \|\nabla f(\bx_t) \|^2} \right)$ for conciseness.
	
	Rearranging the terms, we have 
	\begin{equation}
	\label{eq:rhs_ready}
	\begin{aligned}
	& \sum_{t=1}^{T} \langle \bta_t, \nabla f(\bx_t)^2\rangle \\
	& \leq \frac{1}{1- \frac{4\alpha M (3- \mu)}{\sqrt{\beta} (1- \mu)}} \left[2(f(\bx_1) - f^{\star})+ \frac{M(3-\mu)}{1- \mu} \left(K
	+ \frac{4d\alpha^2\sigma^2}{\beta} \ln \frac{3Te}{\delta} \right)+  \frac{3\| \bta_1 \| \sigma^2 (1- \mu^T)^2}{(1- \mu)^2} \ln \frac{3}{\delta} \right]\\
	& \leq 4(f(\bx_1) - f^{\star})+ \frac{2M(3-\mu)}{1- \mu} \left(K 
	+ \frac{4d\alpha^2\sigma^2}{\beta} \ln \frac{3Te}{\delta} \right)+ \frac{3\| \bta_1 \| \sigma^2 (1- \mu^T)^2}{(1- \mu)^2} \ln \frac{3}{\delta}\\
	& \triangleq C(T), 
	\end{aligned}
	\end{equation}
	where in the second inequality we used $4 \alpha \leq \frac{ \sqrt{\beta}(1-\mu)}{2M(3- \mu)}$. 
	Also, we have
	\begin{align*}
	\sum_{t=1}^{T} \langle \bta_t, \nabla f(\bx_t)^2\rangle
	& \geq \sum_{t=1}^{T} \langle \bta_T, \nabla f(\bx_t)^2\rangle\\
	& = \sum_{i=1}^d \frac{ \alpha \sum_{t=1}^T \nabla f(\bx_t)_i^2}{ \sqrt{ \beta + \sum_{t=1}^T g_{t,i}^2 } }\\
	& \geq \sum_{i=1}^d \frac{ \alpha \sum_{t=1}^T \nabla f(\bx_t)_i^2}{ \sqrt{ \beta + 2 \sum_{t=1}^T \nabla f(\bx_t)_i^2  + 2\sum_{t=1}^T (g_{t,i}-\nabla f(\bx_t)_i)^2 }} \\
	& \geq \sum_{i=1}^d \frac{ \alpha \sum_{t=1}^T \nabla f(\bx_t)_i^2}{\sqrt{ \beta + 2 \sum_{i=1}^d \sum_{t=1}^T \nabla f(\bx_t)_i^2  + 2\sum_{i=1}^d \sum_{t=1}^T (g_{t,i}-\nabla f(\bx_t)_i)^2 } } \\
	&  \geq \frac{ \alpha \sum_{t=1}^T \| \nabla f(\bx_t) \|^2}{ \sqrt{\beta + 2 \sum_{t=1}^T \| \nabla f(\bx_t) \|^2 + 2 T\max_{1\leq t \leq T} \| \bg_t - \nabla f(\bx_t) \|^2}}~.
	\end{align*}
	By Lemma~\ref{lemma:max_bound}, with probability at least $1- \delta$, we have 
	\begin{equation}
	\label{eq:to_solve}
	\sum_{t=1}^T \| \nabla f(\bx_t) \|^2  
	\leq \frac{C(T)}{\alpha} \times \sqrt{ \beta + 2\sum_{t=1}^T \| \nabla f(\bx_t) \|^2 + 2T \sigma^2 \ln \frac{3Te}{\delta} }~. 
	\end{equation}
	\begin{align}
	\text{RHS of \eqref{eq:to_solve}}
	& \leq \frac{C(T)}{\alpha} \times \left( \sqrt{ \beta + 2T\sigma^2 \ln \frac{3Te}{\delta}} + \sqrt{2\sum_{t=1}^T \| \nabla f(\bx_t) \|^2}  \right) \nonumber\\
	& \leq \left[ C+ D\ln \left(A + B\sqrt{ \sum_{t=1}^{T} \| \nabla f(\bx_t) \|^2 }\right) \right] \times \left(A + B \sqrt{ \sum_{t=1}^{T} \| \nabla f(\bx_t) \|^2 }\right), \label{eq:logx}
	\end{align}
	where $A = \sqrt{\beta + 2T \sigma^2 \ln \frac{3Te}{\delta}}$, $B = \sqrt{2}$, 
	$C = \frac{4(f(\bx_1) - f^{\star})}{\alpha}
	+ \frac{8M(3- \mu) d \alpha \sigma^2}{\beta(1- \mu)} \ln \frac{3Te}{\delta} + \frac{3d (1- \mu^T)^2\sigma^2}{\beta (1- \mu)^2} \ln \frac{3}{\delta}$ and $D = \frac{4 \alpha dM (3-\mu)}{1- \mu}$. Using Lemma~\ref{lemma:logsolvex}, we have that 
	\begin{equation}
	\sqrt{\sum_{t=1}^T \| \nabla f(\bx_t) \|^2 } \leq 32B^3D^2 + 2BC + 8B^2D\sqrt{C} + \frac{A}{B}~. 
	\end{equation}
	We use this upper bound in the logarithmic term of \eqref{eq:logx}. Thus, we have \eqref{eq:to_solve} again, this time with 
	\begin{align*}
	C(T) 
	& = C+ D\ln (2A + 32B^4 D^2 + 2B^2C + 8B^3D\sqrt{C} )\\
	& = O\left(\frac{1}{\alpha} + \frac{d \left(\alpha+ \sigma^2 \left(\alpha \ln \frac{T}{\delta} + \frac{\ln \frac{1}{\delta}}{1- \mu}\right)\right)}{1- \mu}\right)~.
	\end{align*}
	Solving \eqref{eq:logx} by Lemma~\ref{lemma:solve_x} and lower bounding $\sum_{t=1}^T \| \nabla f(\bx_t) \|^2 $ by $T \min_{1\leq t \leq T} \| \nabla f(\bx_t) \|^2$, we get the stated bound. 
\end{proof}

\end{document}